\documentclass[pmlr]{jmlr}

\usepackage{longtable}

\usepackage{booktabs}
\usepackage[load-configurations=version-1]{siunitx} 
 \usepackage{booktabs}       
\usepackage{amsfonts}       
\usepackage{nicefrac}       
\usepackage{microtype}      
\usepackage{hyperref}

\usepackage{epstopdf} 

\usepackage{algorithm}
\usepackage[noend]{algorithmic}

\usepackage{float}

\usepackage{lipsum}
\usepackage{amssymb,amsmath}
\usepackage{mathtools}
\usepackage{url}
\usepackage{stmaryrd}
\usepackage{mathtools}

 \usepackage{relsize}

\usepackage{etoolbox}
\usepackage{xparse}

\newcommand{\R}{\Rbb}
\newcommand{\e}{\ten}

\renewcommand{\vec}[1]{\ensuremath{\mathbf{#1}}}
\newcommand{\vecs}[1]{\ensuremath{\mathbf{\boldsymbol{#1}}}}
\newcommand{\mat}[1]{\ensuremath{\mathbf{#1}}}
\newcommand{\mats}[1]{\ensuremath{\mathbf{\boldsymbol{#1}}}}
\newcommand{\ten}[1]{\mat{\ensuremath{\boldsymbol{\mathcal{#1}}}}}

\newcommand{\ttm}[1]{\times_{#1}}
\newcommand{\ttv}[1]{\bullet_{#1}}
\newcommand{\tenmat}[2]{\ten{#1}_{(#2)}}

\newcommand{\tenmatgen}[2]{{(#1)}_{\langle\!\langle #2\rangle\!\rangle}}

\newcommand{\TT}[1]{\llbracket #1 \rrbracket}

\usepackage{pgffor}

\foreach \x in {A,...,Z}{%
\expandafter\xdef\csname \x bb\endcsname{\noexpand\ensuremath{\noexpand\mathbb{\x}}}
}

\foreach \x in {A,...,Z}{%
\expandafter\xdef\csname \x cal\endcsname{\noexpand\ensuremath{\noexpand\mathcal{\x}}}
}

\foreach \x in {A,...,Z}{%
\expandafter\xdef\csname \x t\endcsname{\noexpand\ensuremath{\noexpand\ten{\x}}}
}

\foreach \x in {A,...,Z}{%
\expandafter\xdef\csname \x b\endcsname{\noexpand\ensuremath{\noexpand\mat{\x}}}
}

\foreach \x in {a,...,r}{%
\expandafter\xdef\csname \x b\endcsname{\noexpand\ensuremath{\noexpand\vec{\x}}}
}
\foreach \x in {t,...,z}{%
\expandafter\xdef\csname \x b\endcsname{\noexpand\ensuremath{\noexpand\vec{\x}}}
}

\newcommand{\A}{\mat{A}}
\newcommand{\B}{\mat{B}}

\newcommand{\Hten}{\ten{H}}
\newcommand{\T}{\ten{T}}

\newcommand{\G}{\ten{G}}
\newcommand{\Y}{\ten{Y}}

\newcommand{\X}{\mat{X}}

\renewcommand{\P}{\mat{P}}
\renewcommand{\S}{\mat{S}}

\newcommand{\x}{\vec{x}}

\newtheorem*{theorem*}{Theorem}
\newtheorem*{corollary*}{Corollary}%
\newtheorem*{proposition*}{Proposition}%
\ifcsdef{theorem}{}{
\newtheorem{theorem}{Theorem}[section]

\newtheorem{definition}{Definition}%
}
\newtheorem*{pbm*}{Problem}%
\newtheorem*{algo*}{Algorithm}%

\newcommand{\vectorize}[1]{\mathrm{vec}(#1)}

\DeclareMathOperator*{\rank}{rank}
\newcommand{\norm}[1]{\|#1\|}

\newcommand{\pinv}{^\dagger}

\newcommand{\nstates}{n}

\newcommand{\szerosymbol}{\alpha}
\newcommand{\szero}{\vecs{\szerosymbol}}

\newcommand{\sinfsymbol}{\omega}
\newcommand{\sinf}{\vecs{\sinfsymbol}}

\DeclareDocumentCommand{\wa}{  O{A} O{\szero} O{\sinf} }%
{(#2,\{\mat{#1}^\sigma\}_{\sigma\in\Sigma},#3)}
\DeclareDocumentCommand{\waR}{  O{A} O{\Rbb^\nstates} O{\szero} O{\sinf} }%
{(#2,#3,\{\mat{#1}^\sigma\}_{\sigma\in\Sigma},#4)}

\newcommand{\vvsinfsymbol}{\Omega}
\newcommand{\vvsinf}{\mats{\vvsinfsymbol}}
\DeclareDocumentCommand{\vvwa}{  O{A} O{\szero} O{\vvsinf} }%
{(#2,\{\mat{#1}^\sigma\}_{\sigma\in\Sigma},#3)}

\newcommand{\tzerosymbol}{\alpha}
\newcommand{\tzero}{\vecs{\tzerosymbol}}

\newcommand{\tinfsymbol}{\omega}
\newcommand{\tinf}{\vecs{\tinfsymbol}}

\DeclareDocumentCommand{\wta}{ O{T} O{\Rbb^\nstates} O{\tzero} O{\tinf} O{\Fcal}}%
{(#2,#3,\{\ten{#1}^g\}_{g\in #5_{\geq 1}},\{#4^\sigma\}_{\sigma\in #5_0})}

\DeclareDocumentCommand{\trees}{g}{\IfNoValueTF{#1}{\mathfrak{T}}{\mathfrak{T}_{#1}}}
\DeclareDocumentCommand{\contexts}{g}{\IfNoValueTF{#1}{\mathfrak{C}}{\mathfrak{C}_{#1}}}

\newcommand{\gwmprod}{\diamond}

\DeclareDocumentCommand{\gwm}{  O{M} O{\Fbb^\nstates}}{(#2, \{\ten{#1}^x\}_{x\in\Sigma})}
\DeclareDocumentCommand{\gwmcirc}{  O{M} O{\Rbb^\nstates}}{(#2, \{\mat{#1}^\sigma\}_{\sigma\in\Sigma})}
\DeclareDocumentCommand{\dgwm}{ O{M} O{\Fbb^\nstates}}{(#2, \{\ten{#1}^x\}_{x\in\Sigma},\gwmprod)}

\foreach \x in {A,...,Z}{%
\expandafter\xdef\csname \x bb\endcsname{\noexpand\ensuremath{\noexpand\mathbb{\x}}}
}

\foreach \x in {A,...,Z}{%
\expandafter\xdef\csname \x cal\endcsname{\noexpand\ensuremath{\noexpand\mathcal{\x}}}
}

\usepackage{pgffor}
\foreach \x in {A,...,Z}{%
\expandafter\xdef\csname \x ten\endcsname{\noexpand\ensuremath{\noexpand\ten{\x}}}
}

\foreach \x in {A,...,Z}{%
\expandafter\xdef\csname \x mat\endcsname{\noexpand\ensuremath{\noexpand\mat{\x}}}
}

\foreach \x in {A,...,Z}{%
\expandafter\xdef\csname \x vec\endcsname{\noexpand\ensuremath{\noexpand\mat{\x}}}
}

\newcommand{\ie}{i.e.\ }
\newcommand{\eg}{e.g.\ }
\renewcommand{\e}{\vec{e}}

\usepackage[symbol]{footmisc}

\newcommand{\removelatexerror}{\let\@latex@error\@gobble}

\renewcommand{\P}{\mat{P}}
\renewcommand{\S}{\mat{S}}

\newcommand{\cA}{\mathcal{A}}

\renewcommand{\vec}[1]{\ensuremath{\bm{#1}}}

\usepackage[utf8]{inputenc}
\usepackage{tikz,siunitx}
\usetikzlibrary{positioning}
\usetikzlibrary{calc}

\usetikzlibrary{decorations.pathreplacing}
\usetikzlibrary{cd,arrows,shapes,shapes.misc,intersections}

\usepackage{cancel}
\usepackage{siunitx}
\usepackage{array}
\usepackage{multirow}
\usepackage{amssymb}
\usepackage{gensymb}
\usepackage{mathdots}
\usepackage{graphicx}
\usepackage{bm}
\usepackage{esvect}
\usepackage[algo2e]{algorithm2e}
\allowdisplaybreaks
\usepackage[symbol]{footmisc}

\DeclareMathAlphabet\mathbfcal{OMS}{cmsy}{b}{n}

\theorembodyfont{\upshape}
\theoremheaderfont{\scshape}
\theorempostheader{:}
\theoremsep{\newline}

\jmlrvolume{1}
\jmlryear{2022}
\jmlrworkshop{Learning and Automata Workshop}

\title[Sequential Density Estimation via NCWFAs]{Sequential Density Estimation via \\ Nonlinear Continuous Weighted Finite Automata}

  \author{\Name{Tianyu Li} \Email{tianyu.li@mail.mcgill.ca}\\
  \addr McGill University, MILA
  \AND
  \Name{Bogdan Mazoure} \Email{bogdan.mazoure@mail.mcgill.ca}\\
  \addr McGill University, MILA
  \AND
  \Name{Guillaume Rabusseau} \Email{guillaume.rabusseau@umontreal.ca}\\
  \addr Université de Montréal, DIRO, MILA, CIFAR AI Chair
 }

\begin{document}

\maketitle
\begin{abstract}
Weighted finite automata (WFAs) have been widely applied in many fields. One of the classic problems for WFAs is probability distribution estimation over sequences of discrete symbols. Although WFAs have been extended to deal with continuous input data, namely continuous WFAs (CWFAs)~\citep{li2020connecting}, it is still unclear how to approximate density functions over sequences of continuous random variables using WFA-based models, due to the limitation on the expressiveness of the model as well as the tractability of approximating density functions via CWFAs. In this paper, we propose a nonlinear extension to the CWFA model to first improve its expressiveness, we refer to it as the nonlinear continuous WFAs (NCWFAs). Then we leverage the so-called RNADE method, which is a well-known density estimator based on neural networks, and propose the RNADE-NCWFA model. The RNADE-NCWFA model computes a density function by design. We show that this model is strictly more expressive than the Gaussian HMM model, which CWFA cannot approximate. Empirically, we conduct a synthetic experiment using Gaussian HMM generated data. We focus on evaluating the model's ability to estimate densities for sequences of varying lengths (longer length than the training data). We observe that our model performs the best among the compared baseline methods. 
\end{abstract}
\vspace{-0.25cm}
\section{Introduction}
Many tasks in natural language processing, computational biology, reinforcement learning, and time series analysis rely on learning
with sequential data, \ie, estimating functions defined over sequences of observations from training data.  
Weighted  finite automata~(WFAs) are a powerful and flexible class of models which can efficiently represent such functions.
WFAs are tractable, they encompass a wide range of machine learning models~(they can for example compute any probability distribution defined by a hidden Markov
model~(HMM)~\citep{denis2008rational} and can model the transition and observation behavior of
partially observable Markov decision processes~\citep{thon2015links}) and they offer appealing theoretical guarantees. In particular, 
the so-called 
\emph{spectral
methods} for learning HMMs~\citep{hsu2009spectral}, WFAs~\citep{bailly2009grammatical,balle2014spectral} and related models~\citep{glaude2016pac,boots2011closing}, 
 provide an alternative to Expectation-Maximization (EM) based learning algorithms that is both computationally efficient and 
consistent. 

One of the major applications of WFA is to approximate probability distribution over sequences of discrete symbols. Although the WFA model has been extended to the continuous domain~\citep{li2020connecting,rabusseau2019connecting} as the so-called linear 2-RNN model (or continuous WFA model), approximating density functions for sequential data under continuous domain using this model is not straight-forward, as the model does not guarantee to compute a density function by construction. Moreover, due to the linearity of the model, the continuous WFA model~(CWFA) is not expressive enough to estimate some of the common density functions over sequences of continuous random variables such as a Gaussian hidden Markov model. 

In recent years, neural networks have been widely applied in density estimation and have been proved to be particularly successful. To estimate a density function via neural networks, the neural density estimator need to be flexible enough to represent complex densities but have tractable inference functions and learning algorithms. One particular example of such models is the class of autoregressive models~\citep{uria2016neural, uria2013rnade}, where the joint density is decomposed into a product of conditionals and each conditional is approximated by a neural network. One other type of methods are the so-called flow-based methods~(normalizing flows)~\citep{dinh2014nice, dinh2016density, rezende2015variational}. Flow-based methods transform a base density (e.g. a standard Gaussian) into the target density by an invertible transformation with tractable Jacobian. Although these methods have been used to estimate sequential densities, the sequences often come as fixed length. It is often unclear how to generalize these methods to account for varying length of the sequences in the testing phase, which can be important for some sequential task, such as language modeling for NLP task. Weighted finite automata, on the other hand, are designed to carry out such task under the discrete setting. The question is, how to generalize WFA to approximate density functions over continuous domains. 

In this paper, by extending the classic CWFA model with a (nonlinear) feature mapping and a (nonlinear) termination function, we first propose our nonlinear continuous weighted finite automata (NCWFA) model. Combining this model with the RNADE framework~\citep{uria2013rnade}, we propose RNADE-NCWFA to approximate sequential density functions. The model is flexible as it  naturally generalizes to sequences of varying lengths. Moreover, we show that the RNADE-NCWFA model is strictly more expressive than the Gaussian HMM model. In addition, we propose a spectral learning based algorithm for efficiently learning the parameters of a RNADE-NCWFA. For the empirical study, we conduct synthetic experiments using data generated from a Gaussian HMM model. We compare our proposed spectral learning of  RNADE-NCWFA with HMM learned with the EM algorithm, RNADE with LSTM and RNADE-NCWFA learned with stochastic gradient descent. We evaluate the models' performance through  their log likelihood over sequences of unseen length, meaning the testing sequences are longer than the training sequences, to observe the models' generalization ability. We show that our model outperforms all the baseline models on this metric, especially for long testing sequences. Moreover, the advantage of our model is more significant when dealing with small training sizes and noisy data.


\section{Background}
In this section, we  first introduce  basic tensor algebra. Then we  introduce the continuous weighted finite automata model as well as the RNADE model for density estimation. 
\subsection{Tensor algebra}
We first recall basic definitions of tensor algebra; more details can be found
in~\citep{Kolda09}. 
A {tensor} $\T\in \Rbb^{d_1\times\cdots \times d_p}$ can simply be seen
as a multidimensional array $(\T_{i_1,\cdots,i_p}\ : \ i_n\in [d_n], n\in [p])$. The
{mode-$n$} fibers of $\T$ are the vectors obtained by fixing all
indices except  the $n$th one, \eg $\T_{:,i_2,\cdots,i_p}\in\Rbb^{d_1}$.
The {$n$th mode matricization} of $\T$ is the matrix having the
mode-$n$ fibers of $\T$ for columns and is denoted by
$\tenmat{T}{n}\in \Rbb^{d_n\times d_1\cdots d_{n-1}d_{n+1}\cdots d_p}$.
The vectorization of a tensor is defined by $\vectorize{\T}=\vectorize{\tenmat{T}{1}}$.
In the following $\T$ always denotes a tensor of size $d_1\times\cdots \times d_p$.

The {mode-$n$ matrix product} of the tensor $\T$ and a matrix
$\X\in\Rbb^{m\times d_n}$ is a tensor  denoted by $\T\ttm{n}\X$. It is 
of size $d_1\times\cdots \times d_{n-1}\times m \times d_{n+1}\times
\cdots \times d_p$ and is defined by the relation 
$\Y = \T\ttm{n}\X \Leftrightarrow \tenmat{Y}{n} = \X\tenmat{T}{n}$.
The {mode-$n$ vector product} of the tensor $\T$ and a vector
$\vec{v}\in\Rbb^{d_n}$ is a tensor defined by $\T\ttv{n}\vec{v} = \T\ttm{n}\vec{v}^\top
\in \Rbb^{d_1\times\cdots \times d_{n-1}\times d_{n+1}\times
\cdots \times d_p}$.
%
%
It is easy to check that the $n$-mode product satisfies $(\T\ttm{n}\mat{A})\ttm{n}\mat{B} = \T\ttm{n}\mat{BA}$
where we assume compatible dimensions of the tensor $\T$ and
the matrices $\A$ and $\B$. Given strictly positive integers $n_1,\cdots, n_k$ satisfying
$\sum_i n_i = p$, we use the notation $\tenmatgen{\T}{n_1,n_2,\cdots,n_k}$ to denote the $k$th order tensor 
obtained by reshaping $\T\in\R^{d_1\times \cdots\times d_p}$ into a tensor\footnote{Note that the specific ordering used to perform matricization, vectorization
and such a reshaping is not relevant as long as it is consistent across all operations.} of size 
$(\prod_{i_1=1}^{n_1} d_{i_1}) \times (\prod_{i_2=1}^{n_2} d_{n_1 + i_2}) \times \cdots \times (\prod_{i_k=1}^{n_k} d_{n_1+\cdots+n_{k-1} + i_k}).$

A rank $R$ {tensor train (TT) decomposition}~\citep{oseledets2011tensor} of a tensor 
$\T\in\R^{d_1\times \cdots\times d_p}$ factorizes $\T$ into the product of $p$ core tensors
$\G_1\in\R^{d_1\times R},\G_2\in\R^{R\times d_2\times R},
\cdots, \G_{p-1}\in\R^{R\times d_{p-1} \times R},
 \G_p \in \R^{R\times d_p}$, and is defined\footnote{The classical definition of the TT-decomposition allows the rank $R$ to be different
for each mode, but this definition is sufficient for the purpose of this paper.} by
$\T_{i_1,\cdots,i_p} =  
(\G_1)_{i_1,:}(\G_2)_{:,i_2,:}\cdots 
 (\G_{p-1})_{:,i_{p-1},:}(\G_p)_{:,i_p}
$
for all indices $i_1\in[d_1],\cdots,i_p\in[d_p]$~(here $(\G_1)_{i_1,:}$ is a row vector, where $[d] = \{1, 2, \cdots, d\}$, $(\G_2)_{:,i_2,:}$ is an $R\times R$ matrix, etc.).  We will use the notation $\T = \TT{\G_1,\cdots,\G_p}$
to denote this product. The name of this decomposition comes from the fact that the tensor $\Tt$ is decomposed into a train of lower-order tensors.

\subsection{Continuous weighted finite automata (CWFAs)}
The concept of continuous weighted finite automata (CWFAs) is a generalization of the classic weighted finite automata model to its continuous input case and is also shown to be equivalent to the linear second-order RNN model~\citep{li2020connecting,rabusseau2019connecting}. 
\begin{definition}
A continuous weighted finite automaton with $k$ states (CWFA) is defined by a tuple $A = \langle \vec{\alpha}, \Aten, \mat{\Omega}\rangle$, where $\vec{\alpha} \in \mathbb{R}^k$ is the initial weight, $\Aten\in \mathbb{R}^{k\times d\times k}$ is the transition tensor, and $\mat{\Omega} \in \mathbb{R}^{k\times p}$ is the termination matrix. Let $(\mathbb{R}^d)^*$ denote the set of sequences of size $d$ real-valued vectors. A CWFA computes the following function $f: (\mathbb{R}^d)^* \rightarrow \mathbb{R}^p$:
\begin{equation}
\label{eq:def_cwfa}
    f(\x_1,\x_2,\cdots, \x_n)= (\At\ttv{1}\szero \ttv{2}\x_1)^\top(\At\ttv{2}\x_2)\cdots(\At\ttv{2}\x_n)\mat{\Omega}.
\end{equation}
\end{definition}

To learn the CWFA model,~\citep{li2020connecting} extend the spectral learning algorithm for the classic WFA model~\citep{mohri2012foundations, balle2014method} to its continuous case. The algorithm relies on the Hankel tensor, which is a generalization of the Hankel matrix. 
\begin{definition}
\label{eq: hankel tensor}
For a function  $f:(\R^d)^*\to\R^p$, its Hankel tensor of length $l$, $\Hten_f^l\in \R^{[d]^l \times p}$ is defined by
$(\Hten_f)^{l}_{i_1, \cdots, i_l,:} = f(\e_{i_1},\cdots,\e_{i_l}),$
where $\e_1,\cdots,\e_d$ denotes the  canonical basis  of $\mathbb{R}^d$.
\end{definition}

In practice, to learn the Hankel tensor, one can use gradient descent to minimize the loss function $\mathcal{L}(\x_1,\cdots,\x_l, \vec{y}, \mathbfcal{H}) = \norm{(\x^{\otimes})^\top \tenmatgen{\mathbfcal{H}}{l,1} - \vec{y}}_F^2,$ over a training dataset. Here $\x_1, \cdots, \x_l$ is an input sequence and $\vec{y}$ the corresponding output, and  $\x^{\otimes} = \x_1\otimes \x_2\otimes\cdots\otimes\x_l$. It is is shown in 
\citep{li2020connecting} that the Hankel tensor of a CWFA with finite states can be parameterized by its tensor train form, i.e. $\mathbfcal{H}^{(l)} = \TT{\G_1,\cdots,\G_{l+1}}$. The spectral learning algorithm for CWFA relies on the following theorem~\citep{li2020connecting} showing how to recover the parameters of CWFA from Hankel tensors of the function it computes. 

\begin{theorem}\label{thm:2RNN-SL}
Let $f:(\R^d)^*\to \R^p$ be a function computed by a minimal linear CWFA  with $n$ hidden units and let
$L$ be an integer such that\footnote{It is worth mentioning that such an integer does not always exist. See \citep{li2020connecting} for more details.} $\rank(\tenmatgen{\Hten^{(2L)}_f}{L,L+1}) = n$. Then, for any $\P\in\R^{d^L\times n}$ and $\S\in\R^{n\times d^Lp}$ such that $\tenmatgen{\Hten^{(2L)}_f}{L,L+1} = \P\S,$ the  minimal CWFA computing $f$ is defined by:
$$\szero = (\S\pinv)^\top\tenmatgen{\Hten^{(L)}_f}{L+1}, \ \vvsinf^\top = \P\pinv\tenmatgen{\Hten^{(L)}_f}{L,1},\ \Aten = (\tenmatgen{\Hten^{(2L+1)}_f}{L,1,L+1})\ttm{1}\P\pinv\ttm{3}(\S\pinv)^\top$$

\end{theorem}

\subsection{Real-valued neural autoregressive density estimator (RNADE)}
The real-valued neural autoregressive density estimator (RNADE)~\citep{uria2013rnade} is a generalization of the original neural autoregressive density estimator (NADE)~\citep{uria2016neural} to  continuous variables. 
The core idea of RNADE is to estimate the joint density using the chain rule and approximate each conditional density via neural networks, i.e.
\begin{align}
\label{RNADE}
    p(x_1, \cdots, x_n) &= \prod_{i=1}^n p(x_i|x_{<i})~~\mathrm{with}~~ p(x_i|x_{<i})=p_M(x_i|\theta_i),
\end{align}
where $x_{<i}$ denotes all attributes preceding $x_i\in \mathbb{R}$ in a fixed ordering, $p_M$ is a mixture of $m$ Gaussians with parameters $\theta_i = \{\vec{\beta}_i \in \R^{m}, \vec{\mu}_i\in\R^m, \vec{\sigma}_i\in\R^m\}$. Moreover, we have: $p_M(x_i|\theta_i)=\sum_{j=1}^m\vec{\beta}^{j}_i  \mathcal{N}(x_i| \vec{\mu}^{j}_i, \vec{\sigma}^{j}_i)$, where $\vec{\beta}_{i}^{j}$ denotes the $j$th element of $\vec{\beta}_i$, same for $\vec{\mu}_i^j$ and $\vec{\sigma}_i^j$ and $\mathcal{N}(x_i|\vec{\mu}_i^j, \vec{\sigma}_i^j)$ denotes the Gaussian density with mean $\vec{\mu}_i^j$ and standard deviation $\vec{\sigma}_i^j$ evaluated at $x_i$. Note that $\vec{\beta}_i, \vec{\mu}_i, \vec{\sigma}_i$  are functions of $x_{<i}$. These functions are often chosen to be various forms of neural networks. In the classic setting, RNADE with $m$ mixing components and $k$ hidden states has the following update rules: 
\begin{align}
    \vec{h}_i &= g_i(\vec{h}_{i-1}) ,~~~\vec{\beta}_i = \mathrm{softmax}(\mat{V}^{\beta}_i\vec{h}_{i-1} + \vec{b}_i^\beta) \label{eq:state}\\
    \vec{\mu}_i &= \mat{V}^{\mu}_i\vec{h}_{i-1} + \vec{b}_i^\mu \label{eq:mu},~~~~\vec{\sigma}_i = \mathrm{exp}(\mat{V}^{\sigma}_i\vec{h}_{i-1} + \vec{b}_i^\sigma), 
\end{align}
where $\mat{V}^{\beta}_i, \mat{V}^{\mu}_i$ and $\mat{V}^{\sigma}_i$ are $m\times k$ matrices, $\vec{b}_i^\beta, \vec{b}_i^\mu$ and $\vec{b}_i^\sigma$ are vectors of size $m$, and $g_i$ is an update function for the hidden state which is time step dependent~(see~\citep{uria2013rnade} for more details on the specific update functions used in the original RNADE formulation). The softmax function~\citep{bridle1990probabilistic} ensures the mixing weights $\vec{\beta}$ are positive and sum to one and the exponential ensures the variances are positive. 
RNADE is trained to minimize the negative log likelihood: $\mathcal{L}(x_1\cdots x_n, \theta_i) = -\sum_{i=1}^n  \mathrm{log}(p_M(x_i|\theta_i))$ via gradient descent.

\section{Methodology}
To approximate density functions with CWFA, we need to improve the expressivity of the model and  constrain it to compute a valid density function. In this section, we first introduce nonlinear continuous weighted finite automata. Then, we  present RNADE-NCWFA for sequential density approximation, which combines CWFA with the RNADE framework. In the end, we show that RNADE-NCWFA is strictly more expressive than Gaussian HMM and present our spectral learning based algorithm for learning RNADE-NCWFA. 
\subsection{Nonlinear Continuous Weighted Finite Automata (NCWFAs)}
To leverage CWFAs to estimate density functions, we first need to improve the expressivity of the model. We will do so by introducing a nonlinear feature map as well as a nonlinear termination function. We hence propose the nonlinear continuous weighted finite automata (NCWFA) model as the following:
\begin{definition}
A nonlinear continuous weighted finite automaton (NCWFA) is defined by a tuple $\Tilde{A} = \langle \vec{\alpha}, \xi, \phi, \cA\rangle$, where $\vec{\alpha} \in \mathbb{R}^k$ is the initial weight, $\phi: \mathbb{R}^d \rightarrow \mathbb{R}^{d^\prime}$ is the feature map, $\xi: \mathbb{R}^k \rightarrow \mathbb{R}^p$ is the termination function and $\cA\in \mathbb{R}^{k\times d\times k}$ is the transition tensor. Given a sequence $\x_1, \cdots, \x_l$, the function that the NCWFA $\Tilde{A}$ computes is: 
\begin{align}
    \vec{h_0} &= \vec{\alpha},~~~~\vec{h_t} = \cA\bullet_1\vec{h}_{t-1}\bullet_2\phi(\x_{t}),~~~~f_{\Tilde{A}}(\x_1, \cdots, \x_l) = \xi(\vec{h}_{l}).
\end{align}
\end{definition}
One immediate observation is that we can exactly recover the definition of a CWFA  by letting $\phi(\x_i) =\x_i$ and $\xi(\vec{h}) = \vec{h}^\top \mat{\Omega}$.

\subsection{Density Estimation with NCWFAs}
The second problem to tackle is that we need to constrain the NCWFA so that it can tractably compute a density function. In this section, we will leverage the  RNADE method to propose the RNADE-NCWFAs model. The proposed model is flexible and can compute sequential densities of arbitrary sequence length. Moreover, we will show that this model is strictly more expressive than the classic Gaussian HMM model. 

Recall the core idea of RNADE is to estimate the joint density using the chain rule as in~Equation~\ref{RNADE}. Instead of approximating the conditionals via the classic RNADE treatment as in~Equations~\ref{eq:state}, we use an NCWFA $\Tilde{A} = \langle \vec{\alpha}, \xi, \phi, \mathbfcal{A}\rangle$, i.e., $p(\x_i|\x_{<i}) = f_{\Tilde{A}}(\x_1, \cdots, \x_i)$. One key difference with the classic RNADE model is that the state update function is independent of the time step, allowing the model to generalize to sequences of arbitrary lengths. 
However, an NCWFA does not readily compute a density function, as the function does not necessarily integrates to one and the output is non-negative. 
To overcome this issue, we adopt the approach used in RNADE  by constraining the output of the NCWFA to be a mixture of Gaussians with diagonal covariance matrices: 
\vspace{-0.25cm}
\begin{align}
    \phi(\x_i) &= \mathrm{\tanh}(\x_i^\top \mat{W}),~~\vec{h}_i = \mathbfcal{A}\ttv{1}\vec{h}_{i-1}\bullet_2 \phi(\x_{i}),~~\vec{\beta}_i = \mathrm{softmax}(\mat{V}^{\beta}_i\vec{h}_{i-1} + \vec{b}_i^\beta)  \label{eq:phi_RNADE-NCWFA}\\
    \mat{M}_i &= \mathbfcal{V}^{\mu} \ttv{1}\vec{h}_{i-1} + \mat{B}^\mu,~~\mat{\Sigma}_i = \mathrm{exp}(\mathbfcal{V}^{\sigma}\ttv{1}\vec{h}_{i-1} + \mathbfcal{B}^\sigma)\label{eq: mu, sigma}\\
    \xi(\x_i, \vec{h}_{i-1}) &= \sum_{j=1}^m \vec{\beta}_i^{j} \mathcal{N}(\x_i|\mat{M}_i^{j}, \mathrm{diag}(\mat{\Sigma}_i^j)),~~~~f_{\Tilde{A}}(\x_1, \cdots, \x_l) = \xi(\x_l, \vec{h}_{l-1}) \label{eq:xi_RNADE-NCWFA}
\end{align}
where $\vec{h}_0 = \vec{\alpha}, \mathbfcal{V}^{\mu} \in\mathbb{R}^{k\times m\times d}, \mathbfcal{V}^{\sigma} \in\mathbb{R}^{k\times m\times d}, \mat{B}^\mu\in\mathbb{R}^{m\times d}, \mathbfcal{B}^\sigma\in\mathbb{R}^{m\times d}$, $\vec{\mu}_i^j = (\mat{M}_i)_{:, j}\in \R^{d}, \mat{\Sigma}_i^{j} = (\mat{\Sigma}_i)_{:, j} \in\R^{d}$. $\mathrm{diag}$ is defined to be $\mathrm{diag}(\mat{\Sigma}^j_i) = (\mat{\Sigma}^{j}_i \otimes \vec{1})\circ \mat{I}$, where $\circ$ denotes the Hadamard product, $\vec{1}\in \mathbb{R}^d$ is an all one vector and $\mat{I}\in\mathbb{R}^{d\times d}$ is an identity matrix. For simplicity, we let $d^\prime = d$ and approximate each conditional via a mixture of Gaussian with diagonal covariance matrix. This can be changed to a full covariance matrix, should the corresponding assumption (positive semi-definite) of the matrix is satisfied. Note this simplification does not affect the expressiveness of the model, as a GMM with a diagonal covariance matrix is also an universal approximator for densities and can approximate a GMM with a full covariance matrix~\citep{benesty2008springer}, given enough states. Under this definition, it is easy to show that $\prod_{i=1}^l f_{\Tilde{A}}(\x_{\leq i})$ computes the density of the sequence $\x_{\leq l}$, where $\x_{\leq l}$ denotes $\x_1, \cdots, \x_l$. We will refer to this NCWFA model with RNADE structure as RNADE-NCWFA of $k$ states and $m$ mixtures. Note although the definitions of $\vec{\beta}_i$, $\mat{M}_i$ and $\mat{\Sigma}_i$ takes specific forms, in practice, one can use any differentiable function of $\vec{h}_i$ to compute $\vec{\beta}_i$, $\mat{M}_i$ and $\mat{\Sigma}_i$, so long as $\vec{\beta}_i$ sums to one and $\mat{\Sigma}_i$ is positive. 

One natural question to ask is how expressive this model is. We  show in the following theorem (proof in Appendix~\ref{apd:first}) that RNADE-NCWFA is strictly more expressive than Gaussian HMMs, which is well known for sequential modeling~\citep{bilmes1998gentle}. 
\begin{theorem} Given a Gaussian HMM with $k$ states $\eta = \langle \vec{\mu}, \mat{T}, O\rangle$, where $O: \mathbb{R}^k\times\mathbb{R}^d \rightarrow \mathbb{R}^+$ is the Gaussian emission function, $\mu \in \mathbb{R}^k$ is the initial state distribution and $\mat{T}\in [0,1]^k$ is the transition matrix, there exists a $k$ states $k$ mixtures RNADE-NCWFA $\Tilde{A} = \langle \vec{\alpha}, \xi, \phi, \mathbfcal{A}\rangle$ with full covariance matrices such that the density function over all possible trajectories generated by $\eta$ can be computed by $\Tilde{A}$: $p^{\eta}(\vec{o}_1 \cdots \vec{o}_n) = \prod_{i=1}^{n} f_{\Tilde{A}}(\vec{o}_{\leq n})$
for any trajectory $\vec{o}_1 \cdots \vec{o}_n$. Moreover, there exists a RNADE-NCWFA $\Tilde{A}$ such that no Gaussian HMM model can compute its density. 
\end{theorem}
Note that a CWFA cannot compute the density function of a Gaussian mixture. Indeed, the function computed by a CWFA on a sequence of length $1$ is linear in its input, whereas a RNADE-NCWFA associate such an input to a Gaussian density.

To learn RNADE-NCWFA, we want to maximize the likelihood given some training set   $D_l=\{\x^1_{\leq l}, \cdots , \x^N_{\leq l}\}$ of length-$l$ sequences of $d$ dimensional vectors, i.e., $\x^n_{\leq l} = \x^n_1,\cdots, \x^n_l$, where each $\x^n_i\in\R^d$. More specifically, we want to minimize the negative log likelihood function:
$\mathcal{L}(\Tilde{A}, D) = -\sum_{i=1}^N \sum_{j=1}^l \mathrm{log}(f_{\Tilde{A}}(\x^i_{\leq j})).$
One straight-forward solution is to use gradient descent to optimize this loss function. However, as pointed out in~\citep{bengio1994learning}, due to repeated multiplication by the same transition tensor, gradient descent is prone to suffer from the vanishing gradient problem and to fail in capturing long term dependencies. One alternative is the classic spectral learning algorithm for WFAs. Recall that the spectral learning method for CWFA requires to first learn Hankel tensors of length $L$, $2L$ and $2L+1$ and then perform a rank factorization on the learnt Hankel tensor to recover the CWFA parameters~(see~\citep{li2020connecting}). However, due to the nonlinearity added to the model, namely the feature map $\phi$ and the termination function $\xi$, spectral learning alone will not be enough. To circumvent this issue, we present an algorithm jointly leveraging gradient descent and spectral learning. The idea is to first learn the Hankel tensors of various length and the function $\phi$ and $\xi$ using gradient descent. Then we use the spectral learning algorithm to recover the transition tensor and the initial weights. 

Let $\delta$ and $\omega$ denote the parameters of the mappings $\phi$ and $\xi$, respectively~(see Eq.~\ref{eq:phi_RNADE-NCWFA}-\ref{eq:xi_RNADE-NCWFA}), and let $\Hten^{(l)}_{\Tilde{A}} = \TT{\Gten^{(l)}_1,\cdots,\Gten^{(l)}_{l}}$ be the TT form of the Hankel tensor,  where $\Gten^{(l)}_1\in\R^{d\times k}$ and $\Gten^{(l)}_i\in\R^{k\times d\times k}$ for $i = 2, \cdots, l$. The spectral learning method for RNADE-NCWFAs first involves an approximation of the Hankel tensor via minimizing the following loss function: 
\begin{equation}
\label{loss: density loss}
\mathcal{L}(\delta, \omega,\Gten^{(l)}_1,\cdots,\Gten^{(l)}_{l}, D_l) = -\sum_{i=1}^N\sum_{j=1}^l \mathrm{log}\left[\xi\left(\psi(\x^i_{\leq j})^\top \tenmatgen{\TT{\Gten^{(l)}_1,\cdots,\Gten^{(l)}_{j}}}{j,1}\right)\right] 
\end{equation}
where $\psi(\x_{\leq j}) = \phi(\x_1)\otimes\cdots\otimes\phi(\x_j)$.
In this process, we have obtained the Hankel tensors and the parameters of the termination function and the feature map. Then, one can perform a rank factorization on the learned Hankel tensor and recover the rest of the parameters for the RNADE-NCWFA, namely $\vec{\alpha}, \cA$. The detailed algorithm is presented in Algorithm~\ref{alg:NCWFA-SL}.

\begin{algorithm}[H]
   \caption{\texttt{NCWFA-SL}: Spectral Learning of RNADE-NCWFA}
   \label{alg:NCWFA-SL}
\begin{algorithmic}[1]
   \REQUIRE Three training datasets $D_L,D_{2L},D_{2L+1}$ with input sequences of length $L$, $2L$ and $2L+1$ respectively, an encoder $\phi_{\delta}$, a termination function $\xi_{\omega}$ and TT-parameterized Hankel tensors $\Hten^{(L)}_{\Tilde{A}}$, $\Hten^{(2L)}_{\Tilde{A}}$ and $\Hten^{(2L+1)}_{\Tilde{A}}$, learning rate $\gamma$, desired rank $R$
   \WHILE{Model not converging}
   \FOR{$l\in\{L,2L,2L+1\}$}
       \STATE Update $\delta, \omega, \Gten^{(l)}_1, \cdots, \Gten^{(l)}_l$ via gradient descent by minimize the loss function~\ref{loss: density loss}:
       \FOR{$\theta \in \{\delta, \omega, \Gten^{(l)}_1, \cdots, \Gten^{(l)}_l\}$}
       \STATE $$\theta \leftarrow \theta - \gamma \nabla_{\theta} \mathcal{L}(\delta, \omega,\Gten^{(l)}_1,\cdots,\Gten^{(l)}_{l}, D_l)$$
\ENDFOR
   \ENDFOR
   \ENDWHILE
   
   \STATE\label{alg.line.svd} Let $\tenmatgen{\Hten^{(2L)}_{\Tilde{A}}}{L,L+1} = \P\S$ be a rank $R$ factorization.
   \STATE Return the RNADE-NCWFA $\Tilde{A} = \langle \vec{\alpha}, \xi_\omega, \phi_\delta, \mathbfcal{A}\rangle$ where 
   \begin{align*} 
    \vec{\alpha} &= (\S\pinv)^\top\tenmatgen{\Hten^{(L)}_{\Tilde{A}}}{L+1},~~~~~~~~\Aten= (\tenmatgen{\Hten^{(2L+1)}_{\Tilde{A}}}{L,1,L+1})\ttm{1}\P\pinv\ttm{3}(\S\pinv)^\top
\end{align*}
\end{algorithmic}
\end{algorithm}
\vspace{-0.5cm}




\section{Experiments}
For the experiments, we conduct a synthetic experiment based on data generated from a random 10-states Gaussian HMM. We sample sequences of length 3, 6 and 7 from the HMM. To evaluate the model's performance on its generalization ability to unseen length of sequence, we sample 1,000 sequences from length 8 to length 400 from the same HMM for the test data. To test the model's resistance to noise, we inject the training samples with Gaussian noise of different standard deviations  (0.1 and 1.0) with 0 mean.

For the baseline models, we have HMM learned with expectation maximization (EM) method, as it can computes density of sequences of any length by design.  We also modified the RNADE model by replacing the hidden states update rule~\ref{eq:state} with an LSTM structure to give the RNADE model the ability to generalize to sequences of arbitrary length, regardless of the length of the training sequences. We refer to this model as RNADE-LSTM. For our model, by following Algorithm~\ref{alg:NCWFA-SL}, we have the method RNADE-NCWFA (spec). Alternatively, although we have mentioned that training the (RNADE-)NCWFA model through pure gradient descent method can have many issues, we also list this approach of training RNADE-NCWFA as one of the baselines, referred to as RNADE-NCWFA (sgd). 
For all the training processes, if gradient descent is involved, we always use Adam optimizer~\citep{kingma2014adam} with 0.001 learning rate with early stopping. For HMM as well as RNADE-NCWFA models, we set the size of the model to be 10 (ground truth of the randomly HMM). For RNADE-LSTM, we set the size of the hidden states to be 10. We present the trend of the averaged log likelihood ratio with the ground truth likelihood, i.e. $\mathrm{log}(\frac{\mathrm{predicted~likelihood}}{\mathrm{ground~truth}})$ w.r.t. length of the sequences over 10 seeds in Figure~\ref{fig:foobar} and a snapshot of the log likelihood for each model of 400 length testing sequences in Appendix~\ref{apd:second}. 

From the experiment results, we can see that RNADE-CWFA (spec) consistently has the best performance across all training sizes and levels of noise injected. More precisely, this advantage is more significant when given small training sizes and (or) the data is injected with high noise. Moreover, the spectral learning algorithm shows stable training results as the standard deviation of the log likelihood (ratio) is the lowest among all methods. This is especially the case when not enough training samples are provided. In addition, one can see that this advantage is consistent with all test sequence lengths we have experimented. 
\begin{figure}[t]
    \vspace{-0.5cm}
    \centering
    \includegraphics[width=\textwidth]{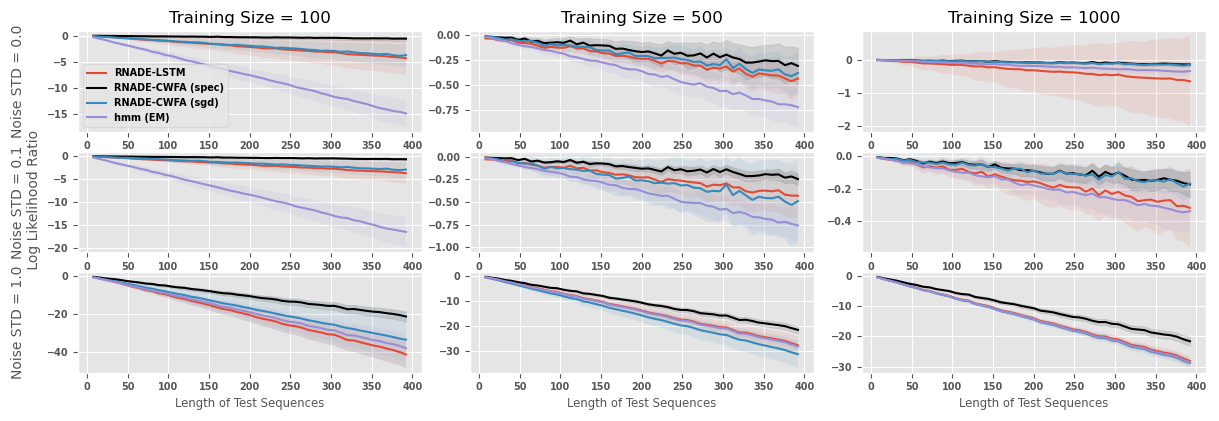}
    \vspace{-0.25cm}
    \caption{Log likelihood ratios between the tested models and the ground truth likelihood. We show the trend w.r.t. the length of the testing sequences under different sample sizes (columns) and standard deviations of the injected noise (rows).}
    \label{fig:foobar}
\vspace{-1cm}
\end{figure}

\vspace{-0.25cm}
\section{Conclusion and Future Work}
In this paper, we propose the RNADE-NCWFA model, an expressive and tractable WFA-based density estimator over sequences of continuous vectors. We extend the notion of continuous WFA to its nonlinear case by introducing a nonlinear feature mapping function as well as a nonlinear termination function. We then combine the idea from RNADE to propose our density estimation model RNADE-NCWFA and its spectral learning based learning algorithm. In addition, we show that theoretically, RNADE-NCWFA is strictly more expressive than the Gaussian HMM model. We show that, empirically, our method has great capability of generalize to sequences of varying length, which is potentially not the same as the training sequences. For future work, we are looking into more experiments on real dataset and compare with more baselines. Moreover, we did not add nonlinear transition for the NCWFA model as it would imply that the Hankel tensor will be of infinite tensor train rank, hence making the spectral learning algorithm intractable. We will be looking into possibilities of adding this nonlinearity into the NCWFA model and have a working algorithm for it. In addition, we would like to examine more closely in terms of the expressivity of the RNADE-NCWFA.


\bibliography{biblio}
\newpage
\appendix

\section{Proof of Theorem}\label{apd:first}
\begin{theorem} Given a Gaussian HMM with $k$ states $\eta = \langle \vec{\mu}, \mat{T}, O\rangle$, where $O: \mathbb{R}^k\times\mathbb{R}^d \rightarrow \mathbb{R}^+$ is the Gaussian emission function, $\mu \in \mathbb{R}^k$ is the initial state distribution and $\mat{T}\in [0,1]^k$ is the transition matrix, there exists a $k$ states $k$ mixtures RNADE-NCWFA $\Tilde{A} = \langle \vec{\alpha}, \xi, \phi, \mathbfcal{A}\rangle$ with full covariance matrices such that the density function over all possible trajectories generated by $\eta$ can be computed by $\Tilde{A}$: 
$$p^{\eta}(\vec{o}_1 \cdots \vec{o}_n) = \prod_{i=1}^{n} f_{\Tilde{A}}(\vec{o}_{\leq n})$$
for any trajectory $\vec{o}_1 \cdots \vec{o}_n$. Moreover, there exists a RNADE-NCWFA $\Tilde{A}$ such that no Gaussian HMM model can compute its density. 
\end{theorem}
\begin{proof}
For the Gaussian HMM $\eta$, given an observation sequences $\vec{o}_1 \cdots \vec{o}_n$, its density under $\eta$ is:
$$p^{\eta}(\vec{o}_1 \cdots \vec{o}_n) = O(\vec{m}^\top, \vec{o}_1)O(\vec{m}^\top \mat{T}, \vec{o}_2)\cdots O(\vec{m}^\top\mat{T}^{n-1}, \vec{o}_n),$$
where $O(\vec{h}, \vec{o}) = \sum_{i=1}^k \vec{h}_i \mathcal{N}(\vec{o}|\vec{\mu}, \mat{\Sigma})$ for some mean vector $\vec{\mu}$ and covariance matrix $\mat{\Sigma}$.
Let $\vec{\alpha} = \vec{m} $, $\cA_{:, i, :} = \mat{T}$ for $i\in [k]$, $\phi(\x) = [\frac{1}{k}, \frac{1}{k}, \cdots, \frac{1}{k}]^\top$ and $\xi = O$. Note it reasonable to let $\xi = O$, since as long as we let $\vec{\beta}_i = \vec{\alpha}^\top \mat{T}^{i-1}$, $\vec{\beta}_0 = \vec{\alpha}^\top$, $\vec{\mu}_i = \vec{\mu}$ and $\vec{\Sigma}_i = \vec{\Sigma}$, following equations~\ref{eq: mu, sigma}, then for any $\vec{h}\in\R^k, \vec{o} \in \R^d$, we have $\xi(\vec{h}, \vec{o}) = O(\vec{h}, \vec{o})$. Then under this parameterization, we have $\cA\ttv{2}\phi(\vec{o}_j) = \mat{T}$. Then the RNADE-NCWFA computes the following function:
\begin{align*}
    f_{\Tilde{A}}(\vec{o}_1, \cdots, \vec{o}_i) &= \xi( (\cA \ttv{1} \vec{\alpha}^\top\ttv{2}\phi(\vec{o}_1))^\top (\cA \ttv{2}\phi(\vec{o}_2))\cdots (\cA \ttv{2}\phi(\vec{o}_{i-1})), \vec{o}_i)\\
    &= \xi(\vec{\alpha}^\top \mat{T}^{i-1}, \vec{o}_i) = O(\vec{m}^\top \mat{T}^{i-1}, \vec{o}_i)
\end{align*}
    
Therefore, we have: 
\begin{align*}
    p^{\eta}(\vec{o}_1 \cdots \vec{o}_n) &= O(\vec{m}^\top, \vec{o}_1)O(\vec{m}^\top \mat{T}, \vec{o}_2)\cdots O(\vec{m}^\top\mat{T}^{n-1}, \vec{o}_n)\\
    &= f_{\Tilde{A}}(\vec{o}_1)f_{\Tilde{A}}(\vec{o}_1, \vec{o}_2)\cdots f_{\Tilde{A}}(\vec{o}_1, \cdots, \vec{o}_n) = \prod_{i=1}^{n} f_{\Tilde{A}}(\vec{o}_{\leq n})
\end{align*}

For the proof of the second half of the theorem, consider a shifting Gaussian HMM, where the mean vector of the Gaussian emission is a function of the time steps, i.e., $\vec{\mu} = q(i)$, where $i = 1, 2, \cdots$. For simplicity, assume the shifting Gaussian HMM is for one dimensional sequences and has one mixture. In addition, let $q(i) = i$ and assume the variance is 1. Then the emission function can be written as $O^t(o) = \mathcal{N}(o|t, 1)$. Then the density of a sequence $o_1, \cdots, o_n$ under this shifting Gaussian HMM $\eta_s$ is:
$$p^{\eta_s}(o_1, \cdots, o_n) = O^1({o}_1)O^2({o}_2)\cdots O^n({o}_n).$$
We show that this density cannot be computed by a Gaussian HMM of finite states. If $p^{\eta_s}$ can be computed by a Gaussian HMM, then for the mean vector $\vec{\mu}$ there exists an initial weight vector $\vec{m}$, a transition matrix $\mat{T}$ satisfying the following linear system: 
\begin{equation*}
    \begin{cases}
    \vec{m}^\top\vec{\mu} &= 1\\
    \vec{m}^\top \mat{T}\vec{\mu} &= 2\\
    &\vdots\\
    \vec{m}^\top \mat{T}^{n-1}\vec{\mu} &= n\\
    &\vdots
\end{cases}
\end{equation*}
This linear system is, however, overdetermined, as $\vec{\mu}$ is a vector of finite size, while there are infinite linearly independent equations to satisfy. Therefore, a Gaussian HMM of finite states cannot compute the density function of a shifting Gaussian HMM. 

We now show such density can be computed by a RNADE-NCWFA. Let $\vec{\alpha}^\top = [1, 1]$, and $\cA_{:, i, :} =  \begin{bmatrix}
1 & 1 \\
0 & 1 
\end{bmatrix}$, $\vec{\mu}_i = \langle \vec{h}_{i-1}, [0, 1] \rangle$, $\phi(o)^\top = [0.5, 0.5]$, $\mat{\Sigma}_i = 1$. Then we have:
\begin{align*}
    f_{\Tilde{A}}({o}_1, \cdots, {o}_i) &= \xi( (\cA \ttv{1} \vec{\alpha}^\top\ttv{2}\phi({o}_1))^\top (\cA \ttv{2}\phi({o}_2))\cdots (\cA \ttv{2}\phi({o}_{i-1})), {o}_i)\\
    &= \xi(\vec{\alpha}^\top \mat{T}^{i-1}, {o}_i) = \xi([1, i], o_i) = \mathcal{N}(o_i|i, 1)
\end{align*}
Therefore:
\begin{align*}
    p^{\eta_s}(o_1, \cdots, o_n) &= \mathcal{N}(o|1, 1)\mathcal{N}(o|1, 2)\cdots\mathcal{N}(o|1, n)\\
    &= f_{\Tilde{A}}({o}_1)f_{\Tilde{A}}({o}_1, {o}_2)\cdots f_{\Tilde{A}}({o}_1, \cdots, {o}_n) = \prod_{i=1}^{n} f_{\Tilde{A}}(\vec{o}_{\leq n})
\end{align*}
Therefore, for the given shifting Gaussian HMM density, it can be computed by a RNADE-NCWFA, but cannot be computed by a Gaussian HMM with finite states. 
\end{proof}

\section{Experiment Results in Table}\label{apd:second}
In this section, we show a snapshot of the results we show in the previous experiment results in Figure~\ref{fig:foobar}. The result is listed in Table~\ref{tab:likelihood}. The reported likelihood is evaluated on test sequence of length 400. From this table, we can see more clearly the advantage of our RNADE-NCWFA model when trained with spectral learning algorithm. 
\begin{table}[h]
\centering
\caption{Comparison of model performance in term of average log likelihood (in NAT). Different models are compared under different training sizes and level of noise injected. The reported likelihood (mean (standard deviation)) is evaluated on test sequence of length 400.}
\label{tab:likelihood}
\resizebox{\textwidth}{!}{%
\begin{tabular}{l|lll|lll|lll}
Training Size & \multicolumn{3}{c|}{100}                                                                                              & \multicolumn{3}{c|}{500}                                                                                               & \multicolumn{3}{c}{1000}                                                                                              \\ \hline
Noise Std     & \multicolumn{1}{c|}{0}                       & \multicolumn{1}{c|}{0.1}                     & \multicolumn{1}{c|}{1}  & \multicolumn{1}{c|}{0}                       & \multicolumn{1}{c|}{0.1}                      & \multicolumn{1}{c|}{1}  & \multicolumn{1}{c|}{0}                       & \multicolumn{1}{c|}{0.1}                     & \multicolumn{1}{c}{1}   \\ \hline
HMM (EM)      & \multicolumn{1}{l|}{-615.26 (2.57)}          & \multicolumn{1}{l|}{-616.88 (3.40)}          & -638.44 (7.50)          & \multicolumn{1}{l|}{-601.15 (0.20)}          & \multicolumn{1}{l|}{-601.18 (0.17)}           & -628.75 (1.51)          & \multicolumn{1}{l|}{-600.70 (0.12)}          & \multicolumn{1}{l|}{-600.69 (0.12)}          & -628.91 (1.13)          \\
RNADE-LSTM    & \multicolumn{1}{l|}{-604.71 (3.36)}          & \multicolumn{1}{l|}{-604.10 (2.29)}          & -641.72 (7.17)          & \multicolumn{1}{l|}{-600.86 (0.15)}          & \multicolumn{1}{l|}{-600.85 (0.23)}           & -628.28 (3.56)          & \multicolumn{1}{l|}{-601.01 (1.34)}          & \multicolumn{1}{l|}{-600.67 (0.24)}          & -628.45 (1.56)          \\
RNADE-NCWFA (spec)           & \multicolumn{1}{l|}{\textbf{-600.96 (0.42)}} & \multicolumn{1}{l|}{\textbf{-601.06 (0.24)}} & \textbf{-621.75 (2.71)} & \multicolumn{1}{l|}{\textbf{-600.73 (0.18)}} & \multicolumn{1}{l|}{\textbf{-600.67 (0.067)}} & \textbf{-622.10 (1.44)} & \multicolumn{1}{l|}{\textbf{-600.50 (0.49)}} & \multicolumn{1}{l|}{-600.53 (0.07)}          & \textbf{-621.91 (1.13)} \\
RNADE-NCWFA (sgd)       & \multicolumn{1}{l|}{-604.11 (2.13)}          & \multicolumn{1}{l|}{-603.29 (1.69)}          & -633.96 (14.6)          & \multicolumn{1}{l|}{-600.81 (0.20)}          & \multicolumn{1}{l|}{-600.91 (0.46)}           & -631.80 (5.53)          & \multicolumn{1}{l|}{-600.51 (0.06)}          & \multicolumn{1}{l|}{\textbf{-600.52 (0.08)}} & -629.11 (1.65)          \\
Ground Truth  & \multicolumn{1}{l|}{-600.40}                 & \multicolumn{1}{l|}{-600.40}                 & -600.40                 & \multicolumn{1}{l|}{-600.40}                 & \multicolumn{1}{l|}{-600.40}                  & -600.40                 & \multicolumn{1}{l|}{-600.40}                 & \multicolumn{1}{l|}{-600.40}                 & -600.40                
\end{tabular}%
}
\end{table}

\end{document}